\documentclass{article}

\usepackage{amsfonts}
\usepackage{amsmath}
\usepackage{amsthm}
\usepackage{amssymb}
\usepackage{bm}
\usepackage{relsize}
\usepackage{cases}
\usepackage[colorlinks=true,linkcolor=blue, citecolor=blue]{hyperref}
\usepackage{breakcites}
\def\ceil#1{\lceil #1 \rceil}
\newtheorem{theorem}{Theorem}
\newtheorem{proposition}{Proposition}

\newcommand{\R}{\mathbb{R}}
\DeclareMathOperator*{\circsum}{\bigcirc}

\DeclareMathOperator*{\argmax}{arg\,max}

\usepackage[style=authoryear]{biblatex}
\usepackage{microtype}
\usepackage{graphicx}
\usepackage{subfigure}
\usepackage{booktabs} 

\begin{document}

\title{Abelian Neural Networks}

\makeatletter
\newcommand{\printfnsymbol}[1]{%
  \textsuperscript{\@fnsymbol{#1}}%
}
\makeatother

\author{
    Kenshin Abe \textsuperscript{\rm{1}},
    Takanori Maehara \textsuperscript{\rm{2}},
    Issei Sato \textsuperscript{\rm{1, 2}} \\
    \textsuperscript{\rm{1}} The University of Tokyo, \textsuperscript{\rm{2}} RIKEN \\
    \texttt{abe.kenshin@gmail.com}
}

\date{}
\maketitle

\begin{abstract}
We study the problem of modeling a binary operation that satisfies some algebraic requirements.
We first construct a neural network architecture for Abelian group operations and derive a universal approximation property.
Then, we extend it to Abelian semigroup operations using the characterization of associative symmetric polynomials.
Both models take advantage of the analytic invertibility of invertible neural networks.
For each case, by repeating the binary operations, we can represent a function for multiset input thanks to the algebraic structure.
Naturally, our multiset architecture has size-generalization ability, which has not been obtained in existing methods.
Further, we present modeling the Abelian group operation itself is useful in a word analogy task.
We train our models over fixed word embeddings and demonstrate improved performance over the original word2vec and another naive learning method.
\end{abstract}

\section{Introduction}
Thanks to the universal approximation theorem \parencite{Cybenko1989ApproximationBS, Leshno1993MultilayerFN, Pinkus1999ApproximationTO, Lu2017TheEP}, feedforward neural networks can approximate any continuous functions.
However, since available data in the real world is limited, the use of a suitable network architecture that reflects inductive biases behind each problem usually leads to better empirical performance on unseen data.
For example, convolutional neural networks \parencite{LeCun1989BackpropagationAT}, which reflect shift invariance of input images, have become the first choice for image recognition;
Invertible neural networks \parencite{Papamakarios2019NormalizingFF}, which are designed to be bijective and be able to compute the inverse, have often been used to model complex probability distributions.
When constructing a new architecture, since we often limit the operations of the networks, its expressive power for a target function class matters.
This has been studied for each case, such as the universality of the convolutional neural networks for continuous functions \parencite{Zhou2018UniversalityOD} and the invertible neural networks for smooth invertible functions \parencite{Teshima2020CouplingbasedIN}.

Recently, there has been increasing attention to permutation invariance as an inductive bias.
Graph neural networks \parencite{Kipf2017SemiSupervisedCW, Gilmer2017NeuralMP} and neural networks for (multi)sets \parencite{Zaheer2017DeepS, Qi2017PointNetDL} reflect this and succeeded in many fields such as chemical molecules, combinatorial optimization, and 3D point clouds.
A byproduct of these models is that they can handle inputs of different sizes, e.g., graphs or (multi)sets of different sizes.
In real-world applications, since annotating label for or training on inputs of large size is computationally expensive, we sometimes try to train a model on smaller data in size and make it generalize to larger test data.
Empirically, some studies have reported the size-generalization ability of neural networks for certain tasks \parencite{Khalil2017LearningCO, Abe2019SolvingNP}.
On the other hand, it has been shown that graph neural networks do not naturally generalize to larger graphs than the training graphs \parencite{Yehudai2020OnSG}.

This work presents the multiset learning setting where we can naturally induce the size generalization.
We consider a function over multisets that can be expressed as the composition of binary operations.
In order for the function to be well-defined, the binary operation needs to form an Abelian semigroup.
For example, $\max(\cdot, \cdot)$, $\min(\cdot, \cdot)$, and $+$ are semigroup operations and they indeed compose well-defined multiset functions: the maximum, minimum, and summation.
To model such binary operations, we propose two novel neural network architectures: Abelian group network and Abelian semigroup network that meet the condition of the Abelian group and semigroup, respectively.
We show that the Abelian group network is a universal approximator of smooth Abelian group operations.
By repeating the binary operations, we can construct two multiset architectures that have the size-generalization ability.
Another useful property of the Abelian group network is that it can explicitly compute the inverse element in the Abelian group.
Therefore, it is also suitable for learning a better function that is otherwise heuristically modeled by simple operations such as $+$ and $-$.

\section{Preliminaries and Related Work} \label{chap:prelim}
\subsection{Definitions}
In this section, let us introduce some basic notations and important definitions that will play a key role in this work.
\subsubsection{Notations}
By $\mathbb{N}$, we represent the set of the natural numbers including $0$.
We denote a vector by a bold symbol, e.g., $\bm{x}$.
Let $\bm{x} \in \mathbb{R}^d$ be a $d$-dimensional vector.
We represent the $i$-th element ($1 \leq i \leq d$) of $\bm{x}$ by $x_i$.
For $1 \leq k \leq d$, $\bm{x}_{\leq k} \in \mathbb{R}^k$ is a $k$-dimensional vector $(x_1, \ldots x_k)^T$ and $\bm{x}_{< k} \in \mathbb{R}^{k-1}$ is a $(k - 1)$-dimensional vector $(x_1, \ldots x_{k - 1})^T$.
We denote the elementwise product of two vectors $x, y \in \mathbb{R}^d$ by $x \otimes y$, such that $(x \otimes y)_{i} = x_{i} y_{i}$.
We denote the elementwise division of two vectors $x \in \mathbb{R}^d, y \in (\mathbb{R} \setminus \{0\})^d$ by $x \oslash y$, such that $(x \oslash y)_{i} = x_{i} / y_{i}$.
Unless otherwise noted, $\| \cdot \|$ represents the $L^2$ (Euclidean) norm.

Let $\mathcal{X}$ be the domain of each element.
We denote the set of multisets over $\mathcal{X}$ by $\mathbb{N}^\mathcal{X}$.
We use $\{\bm{x}_1, \ldots \bm{x}_n \} \in \mathbb{N}^\mathcal{X}$ to describe a multiset composed of $\bm{x}_1, \ldots, \bm{x}_n \in \mathcal{X}$ (any confusion with sets is not problematic in this paper).
Addition over multisets is defined as follows: $\{\bm{x}_1, \ldots, \bm{x}_n\} + \{\bm{x}_{n+1}, \ldots, \bm{x}_N\} = \{\bm{x}_1, \ldots, \bm{x}_n, \bm{x}_{n+1}, \ldots, \bm{x}_N \}$.
The cardinality of a multiset is the number of elements with multiplicity and is expressed by $| \cdot |$ e.g., $|\{1, 2, 2, 3\}| = 4$.
A symmetric group $S_n$ is the set of all $n!$ permutations that can be performed on $n$ elements.

\subsubsection{Universality} \label{sec:prelim-universality}
Universality is an important theoretical property of neural networks' expressive power.
Let $\mathcal{M}$ be a model and $\mathcal{F}$ be a class of target functions, both of which are sets of functions $\mathcal{X} \rightarrow \mathcal{Y}$.
The model $\mathcal{M}$ is a sup-universal approximator of $\mathcal{F}$ if for any target function $f^* \in \mathcal{F}$, for any $\epsilon > 0$, and for any compact subset $\mathcal{K} \subset \mathcal{X}$, there exists a function $f \in \mathcal{M}$ such that
\begin{equation}
    \sup_{\bm{x} \in \mathcal{K}} \| f(\bm{x}) - f^*(\bm{x}) \| < \epsilon .
\end{equation}
If not noted otherwise, universality refers to the sup-universal property.

For $p \in [1, \infty)$, $\mathcal{M}$ is an $L^p$-universal approximator of $\mathcal{F}$ if for any target function $f^* \in \mathcal{F}$, for any $\epsilon > 0$, and for any compact subset $\mathcal{K} \subset \mathcal{X}$, there exists a function $f \in \mathcal{M}$ such that
\begin{equation}
    \int_\mathcal{K} \| f(\bm{x}) - f^*(\bm{x}) \|^p d\bm{x} < \epsilon.
\end{equation}
If $\mathcal{M}$ is sup-universal for $\mathcal{F}$, $\mathcal{M}$ is $L^p$-universal for $\mathcal{F}$.
Therefore, $L^p$-universality is a weaker condition of sup-universality.

\subsubsection{Basic Algebra} \label{sec:agroup-definition}
Here, we introduce the basic definition of important algebraic structures in this study.
Let $G$ be a set and $\circ: G \times G \rightarrow G$ be a binary operation.
Below, we review four properties to define Abelian semigroups and groups.
\begin{description}
\item[Associativity]
For any $x, y, z \in G$, $(x \circ y) \circ z = x \circ (y \circ z)$.
\item[Identity Element]
There exists an element $e \in G$, called the identity element, such that for any $x \in G$, $x \circ e = e \circ x = x$.
\item[Inverse Element]
For any $x \in G$, there exists an element $x^{-1} \in G$, called the inverse element of $x$, such that $x \circ x^{-1} = x^{-1} \circ x = e$.
\item[Commutativity]
For any $x, y \in G$, $x \circ y = y \circ x$.
\end{description}

Table \ref{tbl:algebra-def} shows which properties are required in each algebraic structure.
A semigroup only requires associativity to the binary operation.
A group is a semigroup with an identity element and inverse elements.
An Abelian (semi)group is a (semi)group with commutativity.
\begin{table*}[ht]
    \centering
    \caption{Properties required for each algebraic structure.}
    \begin{tabular}{c|cccc}
        & Associativity & Identity & Inverse & Commutativity
        \\
        \hline
        Semigroup & \checkmark & - & - & - \\
        Group & \checkmark & \checkmark & \checkmark & - \\
        Abelian Semigroup & \checkmark & - & - & \checkmark \\
        Abelian Group & \checkmark & \checkmark & \checkmark & \checkmark
    \end{tabular}
    \label{tbl:algebra-def}
\end{table*}

\subsection{Invertible Neural Networks} \label{sec:inn}
Invertible neural networks are neural networks that approximate invertible functions $\mathbb{R}^d \rightarrow \mathbb{R}^d$.
Here, we review some existing studies for multi-dimensional case, i.e., $d \geq 2$, and single-dimensional case, i.e., $d = 1$.

\subsubsection{Normalizing Flows}
Multi-dimensional invertible neural networks have been studied mainly in the context of normalizing flows \parencite{Tabak2010DENSITYEB}, which iteratively apply invertible functions to a simple original probability distribution to express complex probability distributions \parencite{Kobyzev2020NormalizingFA, Papamakarios2019NormalizingFF}.
There have been many variants proposed including residual flows \parencite{Behrmann2019InvertibleRN}, neural ODEs \parencite{Chen2018NeuralOD}, and autoregressive flows \parencite{Kingma2017ImprovedVI}.
Here we review affine coupling flows \parencite{Dinh2015NICENI}, one of the most popular models with parallelizable efficient inverse computation.
Each layer of the affine coupling flows maps $\bm{x} = (x_1, \ldots, x_d) \in \R^d$ to $\bm{y} = (y_1, \ldots, y_d) \in \R^d$ such that
\begin{equation}
    \begin{cases}
        \bm{y}_{\leq k} &= \bm{x}_{\leq k},  \\
        \bm{y}_{> k} &= \bm{x}_{> k} \otimes \exp(\alpha(\bm{x}_{\leq k})) + \beta(\bm{x}_{\leq k}),
    \end{cases}
\end{equation}
where $\exp$ is applied elemntwise and $\alpha, \beta: \mathbb{R}^k \rightarrow \mathbb{R}^{d-k}$ are trainable functions.
The inverse is computed as follows:
\begin{equation}
    \begin{cases}
        \bm{x}_{\leq k} &= \bm{y}_{\leq k},  \\
        \bm{x}_{> k} &= (\bm{y}_{> k} - \beta(\bm{y}_{\leq k})) \otimes \exp(-\alpha(\bm{y}_{\leq k}))
    \end{cases}
\end{equation}
They are used in many successful applications such as NICE \parencite{Dinh2015NICENI}, Real NVP \parencite{Dinh2017DensityEU}, and Glow \parencite{Kingma2018GlowGF}.

Although the normalizing flows have a limited form of transform, they still admit universalities on certain classes of functions \parencite{Teshima2020CouplingbasedIN}.
The affine coupling flows are $L^p$-universal for $C^2$-diffeomorphism.
Some more complex models including deep sigmoidal flows \parencite{Huang2018NeuralAF} and sum-of-squares polynomial flows \parencite{pmlr-v97-jaini19a} are sup-universal for $C^2$-diffeomorphism.

\subsubsection{Single-dimensional Invertible Neural Networks}
For single-dimensional functions, invertibility is equivalent to strict monotonicity.
Monotonic networks \parencite{Sill1997MonotonicN} model strictly monotonic functions.
Let $K$ be a number of groups and $J_k$ be a number of units for the $k$-th group.
A single-dimensional monotonic network $f:\R \rightarrow \R$ is described as follows with parameters $w^{(k, j)}, b^{(k, j)} \in \R$:
\begin{equation} \label{eq:monotonic-net}
    f(x) = \min_{1 \leq k \leq K} \max_{1 \leq j \leq J_k} w^{(k, j)} \cdot x + b^{(k, j)},
\end{equation}
where all the weights $w^{(k, j)}$ are constrained to be positive for increasing monotonicity and negative for decreasing monotonicity.
For example, the following form is used in practice for the increasingly monotonic case:
\begin{equation}
    f(x) = \min_{1 \leq k \leq K} \max_{1 \leq j \leq J_k} \exp(\tilde{w}^{(k, j)}) \cdot x + b^{(k, j)},
\end{equation}
where $\tilde{w}^{(k, j)} \in \R$.
The monotonic networks are a universal approximator for strictly monotonic functions.
Monotonic rational-quadratic transforms \parencite{Durkan2019NeuralSF} are another universal model for the single-dimensional case.

\subsection{Related Work}
Here, we explain related work.
\subsubsection{Algebraic Structures in Neural Networks}
In the literature of deep learning, algebraic structures mainly appear in the context of group invariant/equivariant neural networks.
For image input, some studies tried to incorporate reflection and rotation invariance into convolutional neural networks \parencite{pmlr-v48-cohenc16, Worrall_2017_CVPR}.
Neural networks for (multi)sets \parencite{Zaheer2017DeepS, Qi2017PointNetDL} adopted invariance/equivariance to symmetric group actions.
Recent studies have investigated symmetries invariant/equivariant to more general group actions, such as a subgroup of the symmetric group \parencite{pmlr-v97-maron19a} and sets of symmetric elements \parencite{pmlr-v119-maron20a}.

On the other hand, our work tries to model an Abelian group/semigroup operation itself.

\subsubsection{Inductive Bias and Expressive Power of Neural Networks}
Inductive biases are assumptions on the nature of the data-generating process or the space of solutions in machine learning \parencite{Battaglia2018RelationalIB}.
Many studies have constructed special neural networks that reflect the inductive biases of a given problem setting.
At the same time, since those networks are often composed of limited forms of neural operations, expressive power including universal approximation properties have been studied.

Convolutional layers of convolutional neural networks (CNN) \parencite{LeCun1989BackpropagationAT} are designed to be invariant to the small shift of an input image.
CNN without fully connected layers has been shown to be universal \parencite{Zhou2018UniversalityOD}.
Forcing the network functions to be bijective (=invertible) is also an inductive bias, which we summarized in Section \ref{sec:inn}.
Message passing graph neural networks \parencite{Gilmer2017NeuralMP} such as graph convolutional networks \parencite{Kipf2017SemiSupervisedCW} and graph attention networks \parencite{Vaswani2017AttentionIA} are designed under the assumption that neighboring nodes have similar properties.
They have been shown to have limited expressive power in terms of graph isomorphism \parencite{Xu2019HowPA, Morris2019WeisfeilerAL} and more expressive models have been studied \parencite{Sato2019ApproximationRO, Maron2019OnTU, Keriven2019UniversalIA, Maehara2019ASP}.
For a (multi)set learning problem, DeepSets \parencite{Zaheer2017DeepS} are one of the most popular models with universal approximation property.

\subsubsection{Size Generalization}
Graph neural networks and neural networks for (multi)sets can handle graphs of different sizes, and their size-generalization ability has been empirically shown in some applications such as physical systems \parencite{Battaglia2016} and combinatorial optimization \parencite{Khalil2017LearningCO, Abe2019SolvingNP, Velickovic2020Neural}.
However, from a theoretical perspective, there exist simple tasks on which graph neural networks do not naturally generalize to larger graphs \parencite{Yehudai2020OnSG}.
Recent work has analyzed the extrapolation of graph neural networks trained by gradient descent \parencite{xu2021how}.
There have been few studies on size generalization of (multi)sets probably because of difficulty in analyzing DeepSets for inputs of different sizes.

\section{Proposed Methods} \label{chap:proposed}
Here, we introduce the proposed methods.
First, we describe the motivation for modeling Abelian group and semigroup operations from the perspective of multiset learning.
Next, we propose a model for Abelian group operations and show its universality.
Then, we extend it for the Abelian semigroup by using the characterization of the associative symmetric polynomials.
Finally, we present architectures for multiset input and show the size-generalization ability of the model for Abelian groups.

\subsection{Motivation on Multiset Functions}
Let $\mathcal{X}$ and $\mathcal{Y}$ be Euclidian spaces, i.e., $\mathcal{X} = \mathbb{R}^{d_1}$ and $\mathcal{Y} = \mathbb{R}^{d_2}$.
A function $f: \mathcal{X}^n \rightarrow \mathcal{Y}$ is called permutation invariant if for any $\bm{X} \in \mathcal{X}^n$ and for any permutation $\sigma \in S_n$, $f(\sigma \cdot \bm{X}) = f(\bm{X})$ holds.
This concept can be extended to functions that take vectors of different dimensions.
Namely, a function $f: \bigcup_{k \in \mathbb{N}} \mathcal{X}^k \rightarrow \mathcal{Y}$ is called permutation invariant if for any $k \in \mathbb{N}$, for any $\bm{X} \in \mathcal{X}^k$ and for any permutation $\sigma \in S_k$, $f(\sigma \cdot \bm{X}) = f(\bm{X})$ holds.
When $f: \bigcup_{k \in \mathbb{N}} \mathcal{X}^k \rightarrow \mathcal{Y}$ is permutation invariant, it can be also viewed as a function that takes multisets as input.
For notation simplicity, we use the same variable to express the multiset function: $f: \mathbb{N}^{\mathcal{X}} \rightarrow \mathcal{Y}$.

In this work, we propose to learn a function over multisets that can be represented as the composition of binary operations.
For this function class, size generalization is naturally guaranteed, as we will show in Theorem \ref{thm:size-generalization}.
Below, we present a necessary and sufficient condition for multiset functions that are represented by the composition of binary operations to be well-defined.
\begin{proposition}[Permutation Invariant Conditions for Binary Operation]
Let $f: \bigcup_{k \in \mathbb{N}} \mathcal{X}^k \rightarrow \mathcal{X}$ be a function represented as
\begin{equation}
    f(\bm{X}) = \bm{x}_1 \circ \cdots \circ \bm{x}_n,
\end{equation}
where $\bm{X} = (\bm{x}_1, \ldots, \bm{x}_n) \in \mathcal{X}^n$ and $\circ: \mathcal{X} \times \mathcal{X} \rightarrow \mathcal{X}$ is a binary operation (left-associative).
The function $f$ is invariant if and only if $\circ$ forms an Abelian semigroup, namely, $\circ$ is commutative and associative.
\end{proposition}
\begin{proof}
It is obvious that when $\circ$ is commutative and associative, $f$ is permutation invariant.
Let us consider the case when $f$ is permutation invariant.
$f((\bm{x}_1, \bm{x}_2)) = f((\bm{x}_2, \bm{x}_1))$ leads to $\bm{x}_1 \circ \bm{x}_2 = \bm{x}_2 \circ \bm{x}_1$ (commutativity).
From $f((\bm{x}_1, \bm{x}_2, \bm{x}_3)) = f((\bm{x}_2, \bm{x}_3, \bm{x}_1))$, we have $(\bm{x}_1 \circ \bm{x}_2) \circ \bm{x}_3 = (\bm{x}_2 \circ \bm{x}_3) \circ \bm{x}_1$ and commutativity leads to $(\bm{x}_1 \circ \bm{x}_2) \circ \bm{x}_3 = \bm{x}_1 \circ (\bm{x}_2 \circ \bm{x}_3)$ (associativity).
\end{proof}

When this condition holds, we represent the multiset version of $f: \mathbb{N}^{\mathcal{X}} \rightarrow \mathcal{X}$ as follows by denoting a composition of $\circ$ by $\circsum$:
\begin{equation} \label{eq:sum-circ}
    f(\bm{X}) = \circsum_{\bm{x} \in \bm{X}} \bm{x},
\end{equation}
where $\bm{X} \in \mathbb{N}^{\mathcal{X}}$ is a multiset of $\mathcal{X}$.
On the basis of this proposition, our goal decomposes into learning Abelian semigroup operations over $\mathcal{X}$.
In Section \ref{sec:agn} and \ref{sec:asn}, we propose neural network architectures for Abelian groups and Abelian semigroups.

\subsection{Abelian Group Network} \label{sec:agn}
We present the Abelian group network that models Abelian group operations as follows:
\begin{equation} \label{eq:agn-binary}
    \bm{x} \circ \bm{y} = \phi^{-1}(\phi(\bm{x}) + \phi(\bm{y})),
\end{equation}
where $\phi: \mathcal{X} \rightarrow \mathcal{X}$ is a trainable invertible function, typically modeled by an invertible neural network.

First, we check that this binary operation satisfies the four conditions of the Abelian group in Section \ref{sec:agroup-definition}.
Associativity and commutativity follow from the following proposition shown in Appendix \ref{sec:proof-semigroup-conserve}.
\begin{proposition}[Semigroup Conservation] \label{prop:associative}
Let $\rho: \mathcal{X} \rightarrow \mathcal{X}$ be a bijective function.
When $\ast: \mathcal{X} \times \mathcal{X} \rightarrow \mathcal{X}$ is associative, $\bm{x} \circ \bm{y} = \rho^{-1}(\rho(\bm{x}) \ast \rho(\bm{y}))$ is also assoviative.
Similarly, when $\ast$ is commutative, $\circ$ is commutative.
\end{proposition}
By this proposition, since $+$ is associative and commutative, the Abelian group network is also associative and commutative.
The identity element is
\begin{equation}
    \bm{e} = \phi^{-1}(\bm{0}),
\end{equation}
which satisfies 
\begin{align}
\begin{split}
    \bm{x} \circ \bm{e} &= \bm{x} \circ (\phi^{-1}(\bm{0}))
    \\
    &= \phi^{-1}(\phi(\bm{x}) + \bm{0}) = \bm{x}.
\end{split}
\end{align}
The inverse element of $\bm{x} \in \mathcal{X}$ is
\begin{equation} \label{eq:abelian-group-inv}
    \bm{x}^{-1} = \phi^{-1}(-\phi(\bm{x})),
\end{equation}
which satisfies
\begin{align}
\begin{split}
    \bm{x} \circ \bm{x}^{-1} &= \bm{x} \circ (\phi^{-1}(-\phi(\bm{x})))
    \\
    &= \phi^{-1}(\phi(\bm{x}) - \phi(\bm{x}))
    \\
    &= \phi^{-1}(\bm{0}) = \bm{e}.
\end{split}
\end{align}
It is worth noting that we can analytically compute the inverse function (Equation \ref{eq:abelian-group-inv}).
The experiment in Section \ref{sec:word-analogy} takes advantage of this quality of the Abelian group network.

Next, we present the universality of the Abelian group network.
\begin{theorem}[Universality of Abelian group networks] \label{thm:universality-agn}
Let $\mathcal{X}$ be a Euclidean space.
Abelian group networks are a universal approximator of Abelian Lie group operations over $\mathcal{X}$.
In other words, for any Abelian Lie group operation $\circ: \mathcal{X} \times \mathcal{X} \rightarrow \mathcal{X}$, for any $\epsilon > 0$, and for any compact subset $\mathcal{K} \subset \mathcal{X}$, there exists a binary operation function $\ast: \mathcal{X} \times \mathcal{X} \rightarrow \mathcal{X}$ represented by an Abelian group network such that
\begin{equation}
    \sup_{\bm{x} \in \mathcal{K}, \bm{y} \in \mathcal{K}} \|(\bm{x} \circ \bm{y}) - (\bm{x} \ast \bm{y})\| < \epsilon.
\end{equation}
\end{theorem}
Appendix \ref{sec:proof-universality-agn} provides the proof.
It is based on the theory of the Lie group and the universality of invertible neural networks.

\subsection{Abelian Semigroup Network} \label{sec:asn}
Although the Abelian group network proposed in Section \ref{sec:agn} is universal for smooth group operations, it is not sufficient for approximating an Abelian semigroup operation such as the product over $\R$, i.e., $x \circ y = xy$.
Now we extend the Abelian group network and propose the Abelian semigroup network.
Our idea is to extend $+$ of Equation \ref{eq:agn-binary} to a polynomial.
From Proposition \ref{prop:associative}, Equation \ref{eq:agn-binary} is still a semigroup after we replace $+$ by a polynomial of $\bm{x}$ and $\bm{y}$ as long as the polynomial is associative and symmetric as a binary operation.
We call the polynomials with this property \emph{associative symmetric polynomials}, which are characterized by the following theorem.
Since the original paper only gives a brief explanation, we give detailed proof in Appendix \ref{sec:proof-assoc-sym-poly}.

\begin{theorem}[Characterization of Associative Symmetric Polynomials, Commutative Case of \parencite{yoshida1963some}] \label{theorem:assoc-symmetric-poly}
An associative symmetric polynomial of $x$ and $y$ is one of the following three forms:
\begin{equation}
    x \ast y = \begin{cases}
        \alpha
        \\
        \alpha + x + y
        \\
        \frac{\beta (\beta - 1)}{\gamma} + \beta(x + y) + \gamma x y & (\gamma \neq 0)
        ,
    \end{cases}
\end{equation}
where $\alpha, \beta, \gamma$ are coefficients.
\end{theorem}

By applying this theorem to Proposition \ref{prop:associative} for $\mathcal{X}$ with elementwise product and division, we obtain the following three kinds of Abelian semigroup operations:
\begin{numcases}{\bm{x} \circ \bm{y} =}
    \rho^{-1}(\bm{\alpha}) &
    \label{eq:semigroup-constant}
    \\
    \rho^{-1}\left( \rho(\bm{x}) + \rho(\bm{y}) + \bm{\alpha} \right) &
    \label{eq:semigroup-group}
    \\
    \rho^{-1}(\bm{\beta} \otimes (\bm{\beta} - \bm{1}) \oslash \bm{\gamma} + \bm{\beta} \otimes (\rho(\bm{x}) + \rho(\bm{y}))
    \nonumber \\
    \ \ \ \ \ \ \ \ \ \ \ \ \ \ \ \ \ \ \ \ \ \ \ \ \ \ \ \ \ \ \ \ \ \ \ \ \ \
    + \bm{\gamma} \otimes \rho(\bm{x}) \otimes \rho(\bm{y}) ), &
    \label{eq:semigroup-main}
\end{numcases}
where $\rho: \mathcal{X} \rightarrow \mathcal{X}$ is an invertible function and $\bm{\alpha}, \bm{\beta}, \bm{\gamma} \in \mathcal{X}$ are parameters ($\bm{\gamma}$ is nonzero for all elements in Equation \ref{eq:semigroup-main}).
Equation \ref{eq:semigroup-constant} is a constant case, on which we do not put a focus due to its trivialness.
Equation \ref{eq:semigroup-group} forms a group where $\bm{e} = \rho^{-1}(-\bm{\alpha})$, $\bm{x}^{-1} = \rho^{-1}(-\rho(\bm{x}) - 2 \bm{\alpha})$.
It can be expressed by the Abelian group network with
$\phi(\bm{x}) = \rho(\bm{x}) + \bm{\alpha}$
and
$\phi^{-1}(\bm{x}) = \rho^{-1}(\bm{x} - \bm{\alpha})$
in Equation \ref{eq:agn-binary}.
Equation \ref{eq:semigroup-main} is a semigroup but not a group.
Just using this equation is fine, but we propose a simpler form, as the Abelian semigroup network:
\begin{equation} \label{eq:asn-binary}
    \bm{x} \circ \bm{y} = \phi^{-1}(\phi(\bm{x}) \otimes \phi(\bm{y})),
\end{equation}
where $\phi: \mathcal{X} \rightarrow \mathcal{X}$ is a trainable invertible function typically modeled by an invertible neural network.
This is a special case of Equation \ref{eq:semigroup-main} when $\bm{\beta} = \bm{0}, \bm{\gamma} = \bm{1}$ and therefore is a semigroup.
Conversely, the Abelian semigroup network can express Equation \ref{eq:semigroup-main} by
\begin{equation}
    \phi(\bm{x}) = \bm{\gamma} \otimes \rho(\bm{x}) + \bm{\beta}, \
    \phi^{-1}(\bm{x}) = \rho^{-1}((\bm{x} - \bm{\beta}) \oslash \bm{\gamma}).
\end{equation}
Moreover, the Abelian semigroup network can approximate the Abelian group network: $\phi'^{-1}(\phi'(\bm{x}) + \phi'(\bm{y}))$.
Let $\mathcal{X}_{>0} = \R_{>0}^d$, where $\mathcal{X} = \R^d$.
One construction is approximating a bijective function $\pi: \mathcal{X} \rightarrow \mathcal{X}_{>0}$,
\begin{equation}
    \pi(\bm{x}) = \exp(\phi'(\bm{x})), \
    \pi^{-1}(\bm{x}) = \phi'^{-1}(\log(\bm{x}))
\end{equation}
by $\phi$, where $\exp$ and $\log$ act elementwise.
This is possible in any compact subset of $\mathcal{X}$.
From the previous discussions so far, the Abelian semigroup network can approximate any binary operation which is homeomorphic to an associative symmetric polynomial.
We confirm this fact in the experiments.

\subsection{Multiset Architecture}
So far, we proposed the binary operation architecture of the Abelian group network and Abelian semigroup network represented by Equation \ref{eq:agn-binary} and \ref{eq:asn-binary}, respectively.
By calculating Equation \ref{eq:sum-circ}, we can write the two models for multiset input $\bm{X} = \{\bm{x_1}, \ldots \bm{x_n}\} \in \mathcal{X}$ in simple forms:
\begin{equation} \label{eq:abelian-group-set}
    f(\bm{X}) = \phi^{-1} \left( \sum_{\bm{x} \in \bm{X}} \phi(\bm{x}) \right),
\end{equation}
\begin{equation}
    f(\bm{X}) = \phi^{-1}(\phi(\bm{x}_1) \otimes \cdots \otimes \phi(\bm{x}_n)),
\end{equation}
where the invertible function $\phi: \mathcal{X} \rightarrow \mathcal{X}$ is typically modeled by an invertible neural network.
We can train this network by usual deep learning strategies, e.g., minimizing a loss function with minibatch stochastic gradiant descent from dataset $\{(\bm{X}_i, \bm{y}_i)\}_{i=1}^N$, where $\bm{X}_i \in \mathbb{N}^\mathcal{X}$ and $\bm{y}_i \in \mathcal{X}$.

Now, we consider the size-generalization ability of the multiset architectures.
An intuitive explanation is as follows.
If trained only on multisets of two elements, our models can learn the correct binary operation.
Therefore, they generalize to multisets of larger size.
For the Abelian group network, since the error bound for small multiset propagates in the form of the sum, we can derive the following theorem.

\begin{theorem}[Size Generalization of Abelian Group Networks] \label{thm:size-generalization}
Let $f^* : \mathbb{N}^\mathcal{X} \rightarrow \mathcal{X}$ be a target function expressed by a composition of Abelian semigroup $(\mathcal{X}, \circ)$: $\displaystyle f^*(\bm{X}) = \circsum_{\bm{x} \in \bm{X}} \bm{x}$.
Let $f : \mathbb{N}^\mathcal{X} \rightarrow \mathcal{X}$ be a multiset architecture of the Abelian group network: $f(\bm{X}) = \phi^{-1} \left( \sum_{\bm{x} \in \bm{X}} \phi(\bm{x}) \right)$.
When
\begin{equation}
    \|f(\bm{X}) - f^*(\bm{X})\| < \epsilon
\end{equation}
holds for any $\bm{X} (\in \mathbb{N}^\mathcal{X})$ whose size is smaller than $a (\geq 2)$, then
\begin{equation}
    \left\| f(\bm{X}) - f^*(\bm{X}) \right\| < \frac{\epsilon \left( (a K_1 K_2)^{\ceil{\log_a b}} - 1 \right)} {a K_1 K_2 - 1}
\end{equation}
holds for any $\bm{X} (\in \mathbb{N}^\mathcal{X})$ whose size is $b (\geq a)$, under the condition that the Lipschitz constants of $\phi$ and $\phi^{-1}$ are $K_1$ and $K_2$, respectively.
\end{theorem}

Appendix \ref{sec:proof-size-generalization} provides the proof.
For the Abelian semigroup network, the error bound for small multiset propagates in the form of the product with the values $\phi(\bm{x}_i)$, which prevents us from inducing the bound like above.
However, it still has the size-generalization ability in most real applications where the values are not too large.
We confirm this by an experiment in Section \ref{sec:toy-example}.

\section{Experiments} \label{chap:experiment}
In this section, we describe two experiments on the effectiveness of the proposed architectures.
First, we check the size generalization of our models on synthetic data.
Next, we present a real-world problem that the binary operation architecture of the Abelian group network is useful.
We train word analogy functions over the fixed vectors of word2vec.

\subsection{Common Settings}
We implemented the neural networks in the PyTorch framework \parencite{Paszke2019PyTorchAI} and optimized them using the Adam algorithm \parencite{Kingma2015AdamAM}.
The hyperparameters for each model in each problem were tuned with validation datasets using the Bayesian optimization of the Optuna framework \parencite{Akiba2019OptunaAN}.
The experiments were run on Intel Xeon E5-2695 v4 with NVIDIA Tesla P100 GPU.
See Appendix \ref{sec:experiment-detail} for the detailed settings, such as the model architecture and the range of hyperparameters.

\subsection{Learning Synthetic Data} \label{sec:toy-example}
To check the size generalization over semigroup and group operations on multisets, we trained the models on synthetic data.
The binary operation forms of the examined functions are $x \circ y = x + y, x + y + 1, \sqrt[^3]{x^3 + y^3}$ (group cases) and $x \circ y = xy, x + y + \frac{xy}{2}$ (semigroup cases).

\paragraph{Setup}
For the single-dimensional invertible neural network of the Abelian group network and Abelian semigroup network, we used monotonic networks \parencite{Sill1997MonotonicN}.
We tuned the hyperparameters, the number of groups and the number of units for each group.
As a baseline, we used DeepSets \parencite{Zaheer2017DeepS}, one of the most popular models for (multi)set learning.
It incorporates two multilayer perceptrons (MLP).
We used the same number of hidden layers for the two MLPs and tuned the hyperparameters, the number of layers in each MLP, the middle dimension, and the hidden dimension.
Each model was trained to minimize the mean squared error on a training set.

\paragraph{Data Generation}
As training data, we generated $500$ multisets of size $\{2,3,4\}$ (chosen uniformly random).
All the elements were single-dimensional and selected uniformly at random from $[-5.0, 5.0]$.
A validation data of $100$ multisets were generated from the same distribution.
We prepared two kinds of test data.
One consisted of $100$ multisets drawn from the same distribution as the training and validation data, which we refer to by \emph{small}.
To see the size-generalization ability, the other consisted of $100$ multisets of size $\{10,11,12\}$ (chosen uniformly at random) with the same element distribution, which we refer to by \emph{large}.

\paragraph{Results}
\begin{table}[ht]
    \centering
    \caption{Mean squared error comparison between the models for each function. The upper three operations are groups and the lower two equations are semigroups. Square root of the values are presented. Smaller is better.}
    \begin{tabular}{c|c|c||ccc}
        \toprule
        $x \circ y$ & & DeepSets & AGN & ASN
        \\
        \midrule
        $x + y$ & small & 0.00226 & \textbf{3.63e-7} & 0.0832
        \\
        & large & 0.908 & \textbf{0.0366} & 0.309
        \\
        \hline
        $x + y + 1$ & small & 0.00772 & \textbf{4.17e-7} & 0.136
        \\ 
        & large & 0.0335 & \textbf{0.0132} & 0.956
        \\
        \hline
        $\sqrt[^3]{x^3 + y^3}$ & small & \textbf{0.0844} & 0.284 & 0.427
        \\
        & large & \textbf{0.229} & 0.636 & 1.26
        \\
        \hline
        \hline
        $xy$ & small & 13.0 & 36.7 & \textbf{0.00000295}
        \\
        & large & 28500 & 28390 & \textbf{31.5}
        \\
        \hline
        $x + y + \frac{xy}{2}$ & small & 0.965 & 7.08 & \textbf{0.000660}
        \\
        & large & 194 & 193 & \textbf{1.22}
        \\
        \bottomrule
    \end{tabular}
    \label{tbl:toy-example}
\end{table}

Table \ref{tbl:toy-example} summarizes the results.
For the group functions, all models including the Abelian semigroup network performed well.
This is consistent with the fact that group operations can be approximated by the Abelian semigroup network, as discussed in Section \ref{sec:asn}.
While the Abelian group network was better on the other two cases, DeepSets outperformed the Abelian group network on $\sqrt[^3]{x^3 + y^3}$.
This is possibly due to the optimization of MLPs in DeepSets being easier than monotonic networks in the Abelian group network and Abelian semigroup network.
Invertible neural networks for the single-dimensional case that are easy to optimize are important for future work.
For the semigroup operations, as well as DeepSets, the Abelian group network did not work well.
This is reasonable because these semigroup operations can not be expressed by the Abelian group network.

On the size generalization, although DeepSets worked fairly well, our models worked better.
For example, the Abelian semigroup network was better than DeepSets on \emph{large} of $x+y$ despite being worse on \emph{small};
The Abelian group network had similar results on $xy$ and $x + y + \frac{xy}{2}$.

\subsection{Word Analogies} \label{sec:word-analogy}
The vector representations of words by word2vec \parencite{Mikolov2013EfficientEO, Mikolov2013DistributedRO} trained only on large unlabeled text data are known to capture linear regularities between words.
For example, vec(``king'') $-$ vec(``man'') $+$ vec(``woman'') results in the most similar vector to vec(``queen'').
Formally, for predicting a word $d$ in a relation $a:b = c:d$, the word with the most similar vector to $\bm{b} - \bm{a} + \bm{c}$ (we denote the corresponding vector for each word by using a bold symbol) is selected in terms of the cosine similarity:
\begin{equation}
    \cos(\bm{v_1}, \bm{v_2}) = \frac{\bm{v_1} \cdot \bm{v_2}}{\| \bm{v_1} \| \| \bm{v_2} \|}.
\end{equation}
Usually, the words $a, b, c$ are excluded from the candidate vocabulary, under the assumption that a common word does not appear in one analogy example.
Although this assumption is reasonable in many cases, it prevents us from solving certain problems such as a past tense verb analogy ``do'':``did'' $=$ ``split'':``split'' or a plural noun analogy ``apple'':``apples'' $=$ ``deer'':``deer''.
On the other hand, if we do not exclude the words $a, b, c$ from the candidates, word2vec suffers from severe performance degradation e.g., falling from $73.59\%$ to $20.64\%$ in our preliminary experiment on the Google analogy test set.
This is due to the nature of the word2vec algorithm: the result of the simple arithmetic calculation $\bm{b} - \bm{a} + \bm{c}$ has a high probability of being close to $\bm{b}$ or $\bm{c}$ in the cosine similarity, especially in a high dimensional space.
One approach to mitigate this issue is to use richer functions than addition and subtraction.
We propose to model a word analogy function by $\bm{b} \circ \bm{a}^{-1} \circ \bm{c}$ where $\circ$ is a group.
Then the original calculation $\bm{b} - \bm{a} + \bm{c}$ can be seen as a special case when $\circ = +$.
In this experiment, we trained the Abelian group network from labeled dataset and compared it with the original word2vec and another learning-based approach.

\begin{table*}[ht]
\centering
\caption{Accuracy on bigger analogy test set when we did not exclude $a, b, c$ from the candidates.}
\label{tbl:eval-bats-main}
\begin{tabular}{lrlll}
\toprule
{} &   num &             WV &              WV + MLP &                WV + AGN \\
\midrule
Overall       &  3314 &   177 (5.34\%) &         565 (17.05\%) &  \textbf{690 (20.82\%)} \\
\hline
Inflectional  &   900 &  100 (11.11\%) &         317 (35.22\%) &  \textbf{435 (48.33\%)} \\
Derivational  &   882 &     4 (0.45\%) &           15 (1.70\%) &    \textbf{20 (2.27\%)} \\
Lexicographic &   632 &    52 (8.23\%) &         154 (24.37\%) &  \textbf{172 (27.22\%)} \\
Encyclopedic  &   900 &    21 (2.33\%) &  \textbf{79 (8.78\%)} &             63 (7.00\%) \\
\bottomrule
\end{tabular}
\end{table*}

\begin{table*}[ht]
\centering
\caption{Accuracy on bigger analogy test set when we excluded $a, b, c$ from the candidates.}
\label{tbl:eval-bats-main-exclude-abc}
\begin{tabular}{lrlll}
\toprule
{} &   num &                      WV &       WV + MLP &                 WV + AGN \\
\midrule
Overall       &  3314 &           864 (26.07\%) &  569 (17.17\%) &  \textbf{1065 (32.14\%)} \\
\hline
Inflectional  &   900 &           614 (68.22\%) &  324 (36.00\%) &   \textbf{656 (72.89\%)} \\
Derivational  &   882 &  \textbf{103 (11.68\%)} &    17 (1.93\%) &             98 (11.11\%) \\
Lexicographic &   632 &            83 (13.13\%) &  151 (23.89\%) &   \textbf{205 (32.44\%)} \\
Encyclopedic  &   900 &             64 (7.11\%) &    77 (8.56\%) &   \textbf{106 (11.78\%)} \\
\bottomrule
\end{tabular}
\end{table*}

\paragraph{Word Embedding}
We used a $300$-dimensional word2vec model for $3$ billion words trained on Google News corpus of about $100$ billion words \footnote{https://code.google.com/archive/p/word2vec/}.
We normalized each word embedding by $L^2$ norm, following the implementation of the Gensim framework \parencite{rehurek_lrec}.

\paragraph{Word Analogy Models}
We compared three different models for a word analogy function $f: \R^{300} \times \R^{300} \times \R^{300} \rightarrow \mathbb{R}^{300}$ that takes the vectors of words $a, b, c$ and predicts the vector of a word $d$.
In the original word2vec,
\begin{equation}
    f(\bm{a}, \bm{b}, \bm{c}) = \bm{b} - \bm{a} + \bm{c}.
\end{equation}
As a baseline, we prepared a learning algorithm based on a multilayer perceptron, which we denote by WV + MLP:
\begin{equation}
    f(\bm{a}, \bm{b}, \bm{c}) = \mathrm{MLP}(\bm{b} - \bm{a} + \bm{c}),
\end{equation}
where we train $\mathrm{MLP}: \mathbb{R}^{300} \rightarrow \mathbb{R}^{300}$.
In the proposed method (WV + AGN), $f$ is modeled as follows:
\begin{align}
\begin{split}
    f(\bm{a}, \bm{b}, \bm{c}) &= \bm{b} \circ \bm{a}^{-1} \circ \bm{c}
    \\
    &= \phi^{-1}(\phi(\bm{b} - \bm{a} + \bm{c})).
\end{split}
\end{align}
For the invertible neural network $\phi:\mathbb{R}^{300} \rightarrow \mathbb{R}^{300}$, we adopted the Glow architecture \parencite{Kingma2018GlowGF}.

\paragraph{Setup}
For WV + MLP and WV + AGN, we minimized the loss function:
\begin{equation}
    \mathrm{loss}_f(\bm{a}, \bm{b}, \bm{c}, \bm{d}) = -\cos(f(\bm{a}, \bm{b}, \bm{c}), \bm{d}),
\end{equation}
on the training set.
We measured the accuracy on the test set by calculating the most similar vector to the model output for each word:
\begin{equation}
    \argmax_{\bm{d} \in \mathcal{V}} \cos(f(\bm{a}, \bm{b}, \bm{c}), \bm{d}),
\end{equation}
where $\mathcal{V}$ is the set of all word embeddings in the word2vec model.
For reference, we also tested the case where we removed the words $a, b, c$ from the candidates:
\begin{equation}
    \argmax_{\bm{d} \in \mathcal{V} \setminus \{a, b, c\}} \cos(f(\bm{a}, \bm{b}, \bm{c}), \bm{d}),
\end{equation}

\paragraph{Datasets}
The bigger analogy test set (BATS) \parencite{Rogers2016AnalogybasedDO} consists of $4$ categories, each of which has $10$ smaller subcategories of $50$ unique relations.
We split the bigger analogy test set into a training set ($60\%$), a validation set ($20\%$), and a test set ($20\%$).
First, for each subcategory, we extracted the pairs included in the word2vec vocabularies and randomly split them into the three sets by each ratio.
Then for each set, we generated all the combinations of the pairs for each subcategory and concatenated them among all subcategories.
Some relations contain multiple acceptable candidates, such as \emph{mammal} and \emph{canine} for hypernyms of \emph{dog}.
We used the first candidate for training and accepted any for the test.
Table \ref{tbl:bats-detail} in Appendix \ref{sec:wv-detail} summarizes the explanation and the number of extracted pairs for all subcategories.
Also, we conducted a transfer experiment to the Google analogy test set \parencite{Mikolov2013EfficientEO} in Appendix \ref{sec:wv-detail}.

\paragraph{Results}
Table \ref{tbl:eval-bats-main} summarizes the results on the bigger analogy test set when we used the whole vocabulary.
The proposed method outperformed WV + MLP in all categories except Encyclopedic.
The accuracy comparison when we excluded $a, b, c$ is shown in Table \ref{tbl:eval-bats-main-exclude-abc}.
In this setting, WV + MLP performed poorly compared even with the original WV.
On the other hand, the proposed method still worked better than WV.
We show the full results for each subcategory of the bigger analogy test set in Table \ref{tbl:eval-bats-detail} and \ref{tbl:eval-bats-detail-exclude-abc} in Appendix \ref{sec:wv-detail}.
We explain the results on transferring test to the Google analogy test set in Appendix \ref{sec:wv-detail}.
Overall, we can conclude that while the naive learning approach overfitted to the certain dataset and evaluation criteria, the inductive biases incorporated in the Abelian group network successfully prevented the model from overfitting.

\section{Conclusion and Future Work} \label{chap:conclusion}
In this work, we proposed two novel neural network architectures, the Abelian group network and Abelian semigroup network, and showed their theoretical properties.
To investigate the effectiveness of our models, we conducted two experiments.
The first experiment on synthetic data validated our theories on expressive power and size generalization.
In the second experiment, we presented that the binary operation version of the Abelian group network is useful for modeling a word analogy function.
Our method improved the performance of word2vec, especially when we searched the whole vocabulary for the candidates.

One of the technical obstacles facing our models when it comes to real-world (multi)set problems is that the dimensions of each element of the input (multi)set and the output should be equal.
By direct implementation, our methods can not be used for (multi)set classification problems such as document-category classification, where we need to output a vector of label-number dimension from a (multi)set of high-dimensional vectors.
This issue can be mitigated by using some other techniques such as label embedding \parencite{Weston2011WSABIESU}.
We believe our models with a theoretical size-generalization guarantee give a new insight into the field of (multi)set learning.

Further, modeling a (multi)set function is a fundamental setting that appears not only in direct real-world applications but also as a component of more complex objects such as graphs \parencite{Xu2019HowPA}.
Constructing size-generalizable graph neural networks by using our models as building blocks remains as future research.

Finally, let us summarize some open problems.
Although the Abelian group network is universal, the expressive power of the Abelian semigroup network is unknown as yet.
Also, the theoretical reason behind the good size-generalization performance of DeepSets in the experiment on learning synthetic data remains a topic of investigation.

\section*{Acknowledgements}
We would like to appreciate Takeshi Teshima for providing knowledge on invertible neural networks.

\newpage
\printbibliography

\appendix
\newpage
\onecolumn
\section{Proofs}
\subsection{Proof of Proposition \ref{prop:associative}} \label{sec:proof-semigroup-conserve}
\begin{proof}
Associativity:
\begin{align}
\begin{split}
    (\bm{x} \circ \bm{y}) \circ \bm{z}
    &= \rho^{-1}(\rho(\rho^{-1}(\rho(\bm{x}) \ast \rho(\bm{y}))) \ast \rho(\bm{z})) \\
    &= \rho^{-1}((\rho(\bm{x}) \ast \rho(\bm{y})) \ast \rho(\bm{z})) \\
    &= \rho^{-1}(\rho(\bm{x}) \ast (\rho(\bm{y}) \ast \rho(\bm{z}))) \ (\because \textrm{Associativity of} \ \ast) \\
    &= \rho^{-1}(\rho(\bm{x}) \ast \rho(\rho^{-1}(\rho(\bm{y}) \ast \rho(\bm{z})))) \\
    &= \bm{x} \circ (\bm{y} \circ \bm{z}).
\end{split}
\end{align}
Commutativity:
\begin{align}
\begin{split}
    \bm{y} \circ \bm{x}
    &= \rho^{-1}(\rho(\bm{y}) \ast \rho(\bm{x})) \\
    &= \rho^{-1}(\rho(\bm{x}) \ast \rho(\bm{y})) \ (\because \textrm{Commutativity of} \ \ast) \\
    &= \bm{x} \circ \bm{y}.
\end{split}
\end{align}
\end{proof}

\subsection{Proof of Theorem \ref{thm:universality-agn}} \label{sec:proof-universality-agn}

First, we review the concept of the Lie group.
A Lie group is a group over a manifold in which the group operation $(x,y) \mapsto x \circ y$ and the inverse function $x \mapsto x^{-1}$ are both differentiable.
An Abelian Lie group is a Lie group that satisfies commutativity.
The real numbers $\mathbb{R}$ with the addition $+$ forms a Lie group, which we denote by $(\mathbb{R}, +)$.
Also, the torus $\mathbb{T} = \mathbb{R} / 2 \pi \mathbb{Z}$ with the addition $+$ modulo $2 \pi \mathbb{Z}$ forms a Lie group, which we denote by $(\mathbb{T}, +)$.
It is known that any connected Abelian Lie group is isomorphic to $(\mathbb{R}, +)^k \times (\mathbb{T}, +)^h$ for some $k, h \in \mathbb{N}$ (Section 4.4.2 of \parencite{procesi2007lie}).

Now, we give the proof of Theorem \ref{thm:universality-agn}.
\begin{proof}
We use the fact that any connected Abelian Lie group is isomorphic to $(\mathbb{R}, +)^{k} \times (\mathbb{T}, +)^{h}$ for some $k, h \in \mathbb{N}$.
The Abelian Lie group over $\mathcal{X}$ is the special cases of this and therefore any $\circ$ can be represented as
\begin{equation}
    \bm{x} \circ \bm{y} = \pi^{-1} (\pi(\bm{x}) + \pi(\bm{y})),
\end{equation}
where $\pi: \mathcal{X} \rightarrow \mathcal{X}$ is a homeomorphic function in terms of Lie groups, i.e., $\pi(\cdot)$ and $\pi^{-1}(\cdot)$ are analytic.
Take any $\epsilon > 0$ and a compact subset $\mathcal{K} \subset \mathcal{X}$.
We denote the image of $\mathcal{K} \times \mathcal{K}$ through the function $(\bm{x}, \bm{y}) \mapsto \pi(\bm{x}) + \pi(\bm{y})$ by
$\mathcal{S}' = \{ \pi(\bm{x}) + \pi(\bm{y}) ~|~ \bm{x}, \bm{y} \in \mathcal{K} \}$.
Let
\begin{equation}
    \mathcal{S} = \{ \bm{s} ~|~ \exists \bm{s}' \in \mathcal{S}' ~ s.t. ~ \|\bm{s} - \bm{s}'\| \leq 2 \epsilon \}
\end{equation}
and
\begin{equation}
    \mathcal{K}' = \mathcal{K} \cup \pi^{-1}(\mathcal{S}).
\end{equation}
Then, we have a Lipschitz constant $L > 0$ of $\pi^{-1}$ over $\mathcal{S}$ since $\pi^{-1}$ is continuous and $\mathcal{S}$ is compact.
Also, from the universality of invertible neural networks \parencite{Teshima2020CouplingbasedIN} for the compact set $\mathcal{K}'$, there exists an invertible neural network $\phi: \mathcal{X} \rightarrow \mathcal{X}$ such that for any $\bm{x} \in \mathcal{K}'$
\begin{equation}
    \|\pi(\bm{x}) - \phi(\bm{x})\| < \frac{\epsilon}{2 L + 1}
\end{equation}
and for any $\bm{x}' \in \pi(\mathcal{K}')$
\begin{equation} \label{eq:inv-error}
    \|\pi^{-1}(\bm{x}') - \phi^{-1}(\bm{x}')\| < \frac{\epsilon}{2 L + 1}.
\end{equation}
Then, for any $\bm{x}, \bm{y} \in \mathcal{K}$, $\phi(\bm{x}) + \phi(\bm{y}) \in \mathcal{S} (\subset \pi(\mathcal{K}'))$ because
\begin{align}
\begin{split}
    \|(\pi(\bm{x}) + \pi(\bm{y})) - (\phi(\bm{x}) + \phi(\bm{y})) \|
    &\leq \|\pi(\bm{x}) - \phi(\bm{x})\| + \|\pi(\bm{y}) + \phi(\bm{y})\|
    \\
    &< \frac{2\epsilon}{2L + 1}
    \\
    &< 2\epsilon.
\end{split}
\end{align}
Therefore, for any $\bm{x}, \bm{y} \in \mathcal{X}$, we have from the Lipshitz continuity of $\pi^{-1}$
\begin{align} \label{eq:lipshitz}
\begin{split}
    \|\pi^{-1}(\pi(\bm{x}) + \pi(\bm{y})) - \pi^{-1}(\phi(\bm{x}) + \phi(\bm{y}))\|
    &\leq L\|(\pi(\bm{x}) + \pi(\bm{y})) - (\phi(\bm{x}) + \phi(\bm{y}))\|
    \\
    &< L \cdot \frac{2\epsilon}{2L+1}
    \\
    &= \frac{2L\epsilon}{2L+1}
\end{split}
\end{align}
and from Equation \ref{eq:inv-error}
\begin{equation} \label{eq:inv-error2}
    \|\pi^{-1}(\phi(\bm{x}) + \phi(\bm{y})) - \phi^{-1}(\phi(\bm{x}) + \phi(\bm{y}))\|
    < \frac{\epsilon}{2L+1}.
\end{equation}
From Equation \ref{eq:lipshitz} and \ref{eq:inv-error2}, for any $\bm{x}, \bm{y} \in \mathcal{K}$, we obtain
\begin{align}
\begin{split}
    \|(\bm{x} \circ \bm{y}) - (\bm{x} \ast \bm{y})\|
    &= \|\pi^{-1}(\pi(\bm{x}) + \pi(\bm{y})) - \phi^{-1}(\phi(\bm{x}) + \phi(\bm{y}))\|
    \\ &\leq \|\pi^{-1}(\pi(\bm{x}) + \pi(\bm{y})) - \pi^{-1}(\phi(\bm{x}) + \phi(\bm{y}))\|
    \\ & \hspace{10pt}
    +\|\pi^{-1}(\phi(\bm{x}) + \phi(\bm{y})) - \phi^{-1}(\phi(\bm{x}) + \phi(\bm{y}))\|
    \\ &\leq \frac{2 L \epsilon}{2L+1} + \frac{\epsilon}{2L+1}
    \\ &= \epsilon.
\end{split}
\end{align}
This concludes that Abelian group networks are universal.
\end{proof}

\subsection{Proof of Theorem \ref{theorem:assoc-symmetric-poly}} \label{sec:proof-assoc-sym-poly}
\begin{proof}
First, we prove that associative polynomials are at most first-order for each variable.
Assume that we have a $n$-order ($n \geq 2$) associative polynomial
\begin{equation}
    x \ast y = \sum_{i=0}^n \sum_{j=0}^n \alpha_{i,j} x^i y^j,
\end{equation}
where $\alpha_{i, j} \in \mathbb{R}$ for $0 \leq i,j \leq n$.
Then, we have
\begin{equation} \label{eq:xy-first}
    (x \ast y) \ast z = \sum_{i=0}^n \sum_{j=0}^n \alpha_{i, j} (\sum_{k=0}^n \sum_{l=0}^n \alpha_{k,l} x^k y^l)^i z^j
\end{equation}
and
\begin{equation} \label{eq:yz-first}
    x \ast (y \ast z) = \sum_{i=0}^n \sum_{j=0}^n \alpha_{i, j} x^i (\sum_{k=0}^n \sum_{l=0}^n \alpha_{k,l} y^k z^l)^i.
\end{equation}
Since $\ast$ is associative, these two must form an identity.
By comparing a coefficient of $x^{n^2}$, we obtain
\begin{equation}
    \sum_{j=0}^n \alpha_{n, j} (\sum_{l=0}^n \alpha_{n, l} y^l)^n z^j = 0.
\end{equation}
If we have $0 \leq j' \leq n$ such that $\alpha_{n, j'} \neq 0$, from the coefficient of $z^{j'}$,
\begin{equation}
    (\sum_{l=0}^n \alpha_{n, l} y^l)^n = 0.
\end{equation}
Recursively, we get $\alpha_{n, 0} = \alpha_{n, 1} = \cdots = \alpha_{n, n} = 0$, which leads to contradiction.
Therefore, we now have 
\begin{equation} \label{eq:zero-lower}
    \alpha_{n, 0} = \alpha_{n, 1} = \cdots = \alpha_{n, n} = 0.
\end{equation}
In the same way, we can also prove
\begin{equation} \label{eq:zero-right}
    \alpha_{0, n} = \alpha_{1, n} = \cdots = \alpha_{n, n} = 0.
\end{equation}
From Equation \ref{eq:zero-lower} and \ref{eq:zero-right} for $n \geq 2$, now we know that associative polynomials are at most first-order for each variable.
Therefore, symmetric associative polynomials have the form:
\begin{equation}
    x \ast y = \alpha + \beta (x + y) + \gamma xy.
\end{equation}
Then we have
\begin{equation}
    (x \ast y) \ast z = \alpha + \beta( (\alpha + \beta(x + y) + \gamma x y) + z) + \gamma (\alpha + \beta(x + y) + \gamma x y) z
\end{equation}
and
\begin{equation}
    x \ast (y \ast z) = \alpha + \beta( (\alpha + \beta(x + y) + \gamma x y) + z) + \gamma (\alpha + \beta(x + y) + \gamma x y) z.
\end{equation}
By solving this identity, we obtain
\begin{equation}
    \alpha \gamma = \beta(\beta - 1).
\end{equation}
This condition is equivalent to the associativity of $\ast$.
It decomposes into three cases: $(\gamma=0, \beta=0)$, $(\gamma=0, \beta=1)$, and $\gamma \neq 0$.
For each case, we obtain
\begin{equation}
    x \ast y = \begin{cases}
        \alpha
        \\
        \alpha + x + y
        \\
        \frac{\beta (\beta - 1)}{\gamma} + \beta(x + y) + \gamma x y & (\gamma \neq 0)
        .
    \end{cases}
\end{equation}
\end{proof}

\subsection{Proof of Theorem \ref{thm:size-generalization}} \label{sec:proof-size-generalization}
\begin{proof}
We prove that for any $\bm{X} \in \mathbb{N}^\mathcal{X}$ of size smaller than $b \geq a$,
\begin{equation} \label{eq:size-error}
    \left\| f(\bm{X}) - f^*(\bm{X}) \right\| < \frac{\epsilon \left( (a K_1 K_2)^{\ceil{\log_a b}} - 1 \right)} {a K_1 K_2 - 1}
\end{equation}
by induction on size $b$.
Note that $K_1 K_2 > 1$ because they are the Lipschitz constants of inverse functions.

\paragraph{Base Case}
When $b=a$, Inequality \ref{eq:size-error} holds.

\paragraph{Inductive Step}
We assume Inequality \ref{eq:size-error} holds for size $b' = 1, \ldots, b - 1$.
We divide $\bm{X}$ of size $b$ into balanced $a$ subsets $\bm{X_1}, \ldots ,\bm{X_a}$ so that $\bm{X} = \bm{X_1} + \cdots \bm{X_a}$ and each $|\bm{X_i}| \leq \ceil{\frac{b}{a}}$.
Then,
\begin{align}
\begin{split}
    \left\| f(\bm{X}) - f^*(\bm{X}) \right\|
    &= \left\| \phi^{-1} \left(\sum_{\bm{x} \in \bm{X}} \phi(\bm{x}) \right) - \circsum_{\bm{x} \in \bm{X}} \bm{x} \right\|
    \\
    &= \left\| \phi^{-1} \left(\sum_{i=1}^a \sum_{\bm{x} \in \bm{X_i}} \phi(\bm{x}) \right) - \circsum_{i=1}^a \left( \circsum_{\bm{x} \in \bm{X_i}} \bm{x} \right) \right\|
    \\
    &= \left\| \phi^{-1} \left(\sum_{i=1}^a \phi(f(\bm{X_i})) \right) - f^*\left( \{f^*(\bm{X_1}), \ldots ,f^*(\bm{X_a}) \} \right) \right\|
    \\
    &\leq \left\| \phi^{-1} \left(\sum_{i=1}^a \phi(f(\bm{X_i})) \right) - \phi^{-1} \left(\sum_{i=1}^a \phi(f^*(\bm{X_i})) \right) \right\| +
    \\& \hspace{10pt}
    \left\| \phi^{-1} \left(\sum_{i=1}^a \phi(f^*(\bm{X_i})) \right) - f^*\left( \{f^*(\bm{X_1}), \ldots ,f^*(\bm{X_a}) \} \right) \right\|
    \\ 
    &\leq K_2 \left\| \sum_{i=1}^a \phi(f(\bm{X_i})) - \phi(f^*(\bm{X_i})) \right\| +
    \\& \hspace{10pt}
    \left\| f \left( \{f^*(\bm{X_1}), \ldots ,f^*(\bm{X_a}) \} \right) - f^*\left( \{f^*(\bm{X_1}), \ldots ,f^*(\bm{X_a}) \} \right) \right\|
    \\
    &< K_2 \left( \sum_{i=1}^a \left\| \phi(f(\bm{X_i})) - \phi(f^*(\bm{X_i})) \right\| \right) + \epsilon
    \\
    &\leq K_2 \left( \sum_{i=1}^a K_1 \left\| f(\bm{X_i}) - f^*(\bm{X_i}) \right\| \right) + \epsilon
    \\
    &= K_1 K_2 \left( \sum_{i=1}^a \left\| f(\bm{X_i}) - f^*(\bm{X_i}) \right\| \right) + \epsilon.
\end{split}
\end{align}

From the assumption on size $\ceil{\frac{b}{a}}$, we obtain
\begin{align}
\begin{split}
    \left\| f(\bm{X}) - f^*(\bm{X}) \right\|
    &< a K_1 K_2 \cdot \frac
        {\epsilon((a K_1 K_2)^{\ceil{\log_a \ceil{\frac{b}{a}}}} - 1)}
        {a K_1 K_2 - 1}
    + \epsilon
    \\
    &< a K_1 K_2 \cdot \frac
        {\epsilon((a K_1 K_2)^{\ceil{\log_a \frac{b}{a}}} - 1)}
        {a K_1 K_2 - 1}
    + \epsilon
    \\
    &< \frac{\epsilon \left( (a K_1 K_2)^{\ceil{\log_a b}} - 1 \right)} {a K_1 K_2 - 1},
\end{split}
\end{align}
which establishes the inductive step.
\end{proof}

\section{Experimental Details} \label{sec:experiment-detail}
Here, we explain the detailed setting and further discussion of the experiments that we did not cover in the main part.

\subsection{Learning Synthetic Data}
\paragraph{Model Architecture}
For the implementation of the monotonic networks, we basically followed the Equation \ref{eq:monotonic-net}, except that we added a coefficient term $s \in \R$ which automatically learn the sign of the weights:
\begin{equation}
    f(x) = \min_{1 \leq k \leq K} \max_{1 \leq j \leq J_k} s \cdot \exp(\tilde{w}^{(k, j)}) \cdot x + b^{(k, j)}.
\end{equation}

\paragraph{Hyperparameters}
All networks were trained by the Adam algorithm of $\mathrm{lr}=10^{-3}$, $\mathrm{beta}=(0.9, 0.999)$ for $1000$ epochs with the batch size of $32$.
Hyperparameters of each model were tuned with the validation dataset using the Optuna framework for each function.
For DeepSets, the number of layers for each MLP was selected from $[2, 8]$ and the middle dimension and hidden dimension were selected from $[2, 32]$.
For the Abelian group network and Abelian semigroup network, the number of groups and the number of units in each group were selected from $[2, 32]$.

\subsection{Word Analogies} \label{sec:wv-detail}
\paragraph{Model architecture}
For the invertible neural network for the Abelian group network and Abelian semigroup network, we implemented Glow architectures based on the FrEIA framework \footnote{https://github.com/VLL-HD/FrEIA}.
We stacked Glow coupling layers and random permutation layers of the dimensions in turn.
For each Glow coupling layer, we used three layer feedforward neural networks with a hyperparameter of \text{hidden\_dim}.

\paragraph{Hyperparameters}
All networks were trained by the Adam algorithm of $\mathrm{lr}=10^{-3}$, $\mathrm{beta}=(0.9, 0.999)$ for $100$ epochs with the batch size of $32$.
The hyperparameters of each model were tuned with the validation dataset using the Optuna framework.
For MLP, the number of layers was selected from $[2, 6]$ and the hidden dimension was selected from $[8, 256]$.
For the Abelian group network, the number of layers was selected from $[2, 6]$ and the hidden dimension was selected from $[8, 256]$.
Weight\_decay was selected from $[0, 10^{-3}]$ for all models.

Table \ref{tbl:word-analogy-hyperparam} summarizes the selected hyperparameters for each model.
\begin{table}[ht]
    \centering
    \caption{Selected hyperparameters in word analogy task.}
    \begin{tabular}{c|ccc}
        & layer\_num & hidden\_dim & weight\_decay
        \\
        \hline
        W2V & - & - & -  \\
        W2V + MLP & 4 & 223 & 6.43e-4  \\
        W2V + AGN & 5 & 151 & 1.60e-4 
    \end{tabular}
    \label{tbl:word-analogy-hyperparam}
\end{table}

\paragraph{Transfer Test} \label{sec:word-analogy-dataset-detail}
\begin{table}
\centering
\caption{Detailed explanation of bigger analogy test set. \emph{pair} refers to the whole relation size and \emph{used} refers to the number included in the word2vec model.}
\label{tbl:bats-detail}
\begin{tabular}{lllrr}
\toprule
     category &              subcategory &                 example &  pair &  used \\
\midrule
 Inflectional &    I01 noun - plural\_reg &            album:albums &    50 &    50 \\
 Inflectional &  I02 noun - plural\_irreg &       ability:abilities &    50 &    48 \\
 Inflectional &    I03 adj - comparative &           angry:angrier &    50 &    49 \\
 Inflectional &    I04 adj - superlative &             able:ablest &    50 &    49 \\
 Inflectional &      I05 verb\_inf - 3pSg &          accept:accepts &    50 &    50 \\
 Inflectional &      I06 verb\_inf - Ving &       achieve:achieving &    50 &    49 \\
 Inflectional &       I07 verb\_inf - Ved &         accept:accepted &    50 &    50 \\
 Inflectional &     I08 verb\_Ving - 3pSg &             adding:adds &    50 &    50 \\
 Inflectional &      I09 verb\_Ving - Ved &            adding:added &    50 &    50 \\
 Inflectional &      I10 verb\_3pSg - Ved &              adds:added &    50 &    50 \\
 \hline
 Derivational &        D01 noun+less\_reg &             arm:armless &    50 &    48 \\
 Derivational &           D02 un+adj\_reg &             able:unable &    50 &    49 \\
 Derivational &           D03 adj+ly\_reg & according:accordingl... &    50 &    49 \\
 Derivational &         D04 over+adj\_reg & ambitious:overambiti... &    50 &    50 \\
 Derivational &         D05 adj+ness\_reg &     amazing:amazingness &    50 &    45 \\
 Derivational &          D06 re+verb\_reg &       acquire:reacquire &    50 &    48 \\
 Derivational &        D07 verb+able\_reg &       accept:acceptable &    50 &    49 \\
 Derivational &        D08 verb+er\_irreg &        achieve:achiever &    50 &    49 \\
 Derivational &      D09 verb+tion\_irreg &       accuse:accusation &    50 &    48 \\
 Derivational &      D10 verb+ment\_irreg & accomplish:accomplis... &    50 &    47 \\
 \hline
 Encyclopedic &    E01 country - capital &           abuja:nigeria &    50 &    37 \\
 Encyclopedic &   E02 country - language &         andorra:catalan &    50 &    36 \\
 Encyclopedic &     E03 UK\_city - county & aberdeen:aberdeenshi... &    50 &    24 \\
 Encyclopedic &   E04 name - nationality &         aristotle:greek &    50 &    23 \\
 Encyclopedic &    E05 name - occupation & andersen:writer/poet... &    50 &    27 \\
 Encyclopedic &       E06 animal - young &         ape:baby/infant &    50 &    50 \\
 Encyclopedic &       E07 animal - sound &             alpaca:bray &    50 &    50 \\
 Encyclopedic &     E08 animal - shelter & ant:anthill/insectar... &    50 &    50 \\
 Encyclopedic &       E09 things - color &     ant:black/brown/red &    50 &    50 \\
 Encyclopedic &        E10 male - female &           actor:actress &    50 &    48 \\
 \hline
Lexicographic &  L01 hypernyms - animals & allosaurus:dinosaur/... &    50 &    50 \\
Lexicographic &     L02 hypernyms - misc & armchair:chair/seat/... &    50 &    50 \\
Lexicographic &      L03 hyponyms - misc & backpack:daypack/kit... &    50 &    50 \\
Lexicographic & L04 meronyms - substance & atmosphere:gas/oxyge... &    50 &    50 \\
Lexicographic &    L05 meronyms - member &          acrobat:troupe &    50 &    50 \\
Lexicographic &      L06 meronyms - part & academia:college/uni... &    50 &    47 \\
Lexicographic & L07 synonyms - intensity & afraid:terrified/hor... &    50 &    50 \\
Lexicographic &     L08 synonyms - exact & airplane:aeroplane/p... &    50 &    50 \\
Lexicographic &  L09 antonyms - gradable & able:unable/incapabl... &    50 &    50 \\
Lexicographic &    L10 antonyms - binary & after:before/earlier... &    50 &    50 \\
\bottomrule
\end{tabular}
\end{table}

\begin{table}
\centering
\caption{Detailed explanation of Google analogy test set. num refers to the whole relation size and used refers to the number included in the word2vec model.}
\label{tbl:google-detail}
\begin{tabular}{lllrr}
\toprule
 category &                 subcategory &                 example &  num &  used \\
\midrule
 Semantic &    capital-common-countries &           Athens:Greece &  506 &   506 \\
 Semantic &               capital-world &           Abuja:Nigeria & 4524 &  4524 \\
 Semantic &                    currency &           Algeria:dinar &  866 &   866 \\
 Semantic &               city-in-state &        Chicago:Illinois & 2467 &  2467 \\
 Semantic &                      family &                boy:girl &  506 &   506 \\
\hline
Syntactic &   gram1-adjective-to-adverb &       amazing:amazingly &  992 &   992 \\
Syntactic &              gram2-opposite & acceptable:unacceptable &  812 &   812 \\
Syntactic &           gram3-comparative &               bad:worse & 1332 &  1332 \\
Syntactic &           gram4-superlative &               bad:worst & 1122 &  1122 \\
Syntactic &    gram5-present-participle &             code:coding & 1056 &  1056 \\
Syntactic & gram6-nationality-adjective &        Albania:Albanian & 1599 &  1599 \\
Syntactic &            gram7-past-tense &          dancing:danced & 1560 &  1560 \\
Syntactic &                gram8-plural &          banana:bananas & 1332 &  1332 \\
Syntactic &          gram9-plural-verbs &      decrease:decreases &  870 &   870 \\
\bottomrule
\end{tabular}
\end{table}
Google analogy test set \parencite{Mikolov2013EfficientEO} includes $19,544$ question pairs ($8,869$ semantic and $10,675$ syntactic).
The semantic questions are composed of five categories: common capital city, all capital cities, currency, city-in-state, and man-woman.
The syntactic questions are composed of $9$ categories: adjective to adverb, opposite, comparative, superlative, present participle, nationality adjective, past tense, plural nouns, and plural verbs.
Table \ref{tbl:google-detail} shows the detailed explanation of the Google analogy test set.
All the words were included in the word2vec vocabulary.

To test the transferability to another dataset, we measured the accuracy on the Google analogy test set.
We compared the models trained on the bigger analogy test set (Section \ref{sec:word-analogy}).
Table \ref{tbl:eval-google} summarizes the full results on the Google analogy test set when we used the whole vocabulary as the candidates.
The proposed model trained on the bigger analogy test set successfully transferred to the Google analogy test set with the best accuracy on $11$ subcategories out of $14$.
On the other hand, the performance of WV + MLP significantly deteriorated, especially $0\%$ accuracy in $4$ subcategories out of $5$ in the semantics task.
Table \ref{tbl:eval-google-exclude-abc} shows the results on the Google analogy test set when we excluded $a, b, c$ from the candidates.
In this case, the original WV performed the best.
This is probably because the word2vec model was highly tuned for the Google analogy test set for this evaluation method.

\paragraph{Detailed Results}
\begin{table}
\centering
\caption{Model comparison for each subcategory of Google analogy test set when we did not excluded $a, b, c$ from the candidates.}
\label{tbl:eval-google}
\begin{tabular}{lrlll}
\toprule
{} &   num &                     WV &                  WV + MLP &                     WV + AGN \\
\midrule
Overall   &  19544 &  4033 (20.64\%) &   1346 (6.89\%) &  \textbf{5676 (29.04\%)} \\
Semantic  &   8869 &  1995 (22.49\%) &    161 (1.82\%) &  \textbf{2260 (25.48\%)} \\
Syntactic &  10675 &  2038 (19.09\%) &  1185 (11.10\%) &  \textbf{3416 (32.00\%)} \\
\hline
capital-c... &   506 &            225 (44.47\%) &            0 (0.00\%) &   \textbf{227 (44.86\%)} \\
capital-w... &  4524 &           1168 (25.82\%) &            0 (0.00\%) &  \textbf{1223 (27.03\%)} \\
currency     &   866 &   \textbf{185 (21.36\%)} &            0 (0.00\%) &            119 (13.74\%) \\
city-in-s... &  2467 &            252 (10.21\%) &            0 (0.00\%) &   \textbf{350 (14.19\%)} \\
family       &   506 &            165 (32.61\%) &         161 (31.82\%) &   \textbf{341 (67.39\%)} \\
\hline
gram1-adj... &   992 &              15 (1.51\%) &  \textbf{86 (8.67\%)} &              84 (8.47\%) \\
gram2-opp... &   812 &              14 (1.72\%) &         200 (24.63\%) &   \textbf{235 (28.94\%)} \\
gram3-com... &  1332 &            329 (24.70\%) &         242 (18.17\%) &   \textbf{713 (53.53\%)} \\
gram4-sup... &  1122 &            124 (11.05\%) &         244 (21.75\%) &   \textbf{406 (36.19\%)} \\
gram5-pre... &  1056 &              73 (6.91\%) &           71 (6.72\%) &   \textbf{160 (15.15\%)} \\
gram6-nat... &  1599 &  \textbf{1180 (73.80\%)} &            0 (0.00\%) &            996 (62.29\%) \\
gram7-pas... &  1560 &             134 (8.59\%) &          127 (8.14\%) &   \textbf{353 (22.63\%)} \\
gram8-plu... &  1332 &              63 (4.73\%) &           82 (6.16\%) &   \textbf{176 (13.21\%)} \\
gram9-plu... &   870 &            106 (12.18\%) &         133 (15.29\%) &   \textbf{293 (33.68\%)} \\
\bottomrule
\end{tabular}
\end{table}

\begin{table}
\centering
\caption{Model comparison for each subcategory of Google analogy test set when we excluded $a, b, c$ from the candidates.}
\label{tbl:eval-google-exclude-abc}
\begin{tabular}{lrlll}
\toprule
{} &   num &                       WV &       WV + MLP &                 WV + AGN \\
\midrule
Overall   &  19544 &  \textbf{14382 (73.59\%)} &   1427 (7.30\%) &  11857 (60.67\%) \\
Semantic  &   8869 &   \textbf{6482 (73.09\%)} &    163 (1.84\%) &   4918 (55.45\%) \\
Syntactic &  10675 &   \textbf{7900 (74.00\%)} &  1264 (11.84\%) &   6939 (65.00\%) \\
\hline
capital-common-... &   506 &   \textbf{421 (83.20\%)} &     0 (0.00\%) &            378 (74.70\%) \\
capital-world      &  4524 &  \textbf{3580 (79.13\%)} &     0 (0.00\%) &           2689 (59.44\%) \\
currency           &   866 &   \textbf{304 (35.10\%)} &     0 (0.00\%) &            180 (20.79\%) \\
city-in-state      &  2467 &  \textbf{1749 (70.90\%)} &     0 (0.00\%) &           1223 (49.57\%) \\
family             &   506 &            428 (84.58\%) &  163 (32.21\%) &   \textbf{448 (88.54\%)} \\
\hline
gram1-adjective... &   992 &   \textbf{283 (28.53\%)} &    93 (9.38\%) &   \textbf{283 (28.53\%)} \\
gram2-opposite     &   812 &            347 (42.73\%) &  201 (24.75\%) &   \textbf{419 (51.60\%)} \\
gram3-comparati... &  1332 &  \textbf{1210 (90.84\%)} &  248 (18.62\%) &           1044 (78.38\%) \\
gram4-superlati... &  1122 &   \textbf{980 (87.34\%)} &  246 (21.93\%) &            777 (69.25\%) \\
gram5-present-p... &  1056 &   \textbf{825 (78.12\%)} &    78 (7.39\%) &            729 (69.03\%) \\
gram6-nationali... &  1599 &  \textbf{1438 (89.93\%)} &     0 (0.00\%) &           1158 (72.42\%) \\
gram7-past-tens... &  1560 &           1029 (65.96\%) &  157 (10.06\%) &  \textbf{1061 (68.01\%)} \\
gram8-plural       &  1332 &  \textbf{1197 (89.86\%)} &   103 (7.73\%) &            826 (62.01\%) \\
gram9-plural-ve... &   870 &            591 (67.93\%) &  138 (15.86\%) &   \textbf{642 (73.79\%)} \\
\bottomrule
\end{tabular}
\end{table}
\begin{table}
\centering
\caption{Model comparison for each subcategory of bigger analogy test set when we did not exclude $a, b, c$ from the candidates.}
\label{tbl:eval-bats-detail}
\begin{tabular}{lrlll}
\toprule
{} &  num &    WV &     WV + MLP &     WV + AGN \\
\midrule
I01 &   90 &             2 (2.22\%) &             5 (5.56\%) &    \textbf{7 (7.78\%)} \\
I02 &   90 &    \textbf{0 (0.00\%)} &    \textbf{0 (0.00\%)} &    \textbf{0 (0.00\%)} \\
I03 &   90 &           13 (14.44\%) &           22 (24.44\%) &  \textbf{41 (45.56\%)} \\
I04 &   90 &           10 (11.11\%) &           18 (20.00\%) &  \textbf{39 (43.33\%)} \\
I05 &   90 &           26 (28.89\%) &           58 (64.44\%) &  \textbf{60 (66.67\%)} \\
I06 &   90 &           14 (15.56\%) &           13 (14.44\%) &  \textbf{57 (63.33\%)} \\
I07 &   90 &             3 (3.33\%) &  \textbf{54 (60.00\%)} &           42 (46.67\%) \\
I08 &   90 &           11 (12.22\%) &           40 (44.44\%) &  \textbf{54 (60.00\%)} \\
I09 &   90 &             8 (8.89\%) &           49 (54.44\%) &  \textbf{62 (68.89\%)} \\
I10 &   90 &           13 (14.44\%) &           58 (64.44\%) &  \textbf{73 (81.11\%)} \\
\hline
D01 &   90 &    \textbf{0 (0.00\%)} &    \textbf{0 (0.00\%)} &    \textbf{0 (0.00\%)} \\
D02 &   90 &             0 (0.00\%) &    \textbf{1 (1.11\%)} &             0 (0.00\%) \\
D03 &   90 &             1 (1.11\%) &             2 (2.22\%) &    \textbf{5 (5.56\%)} \\
D04 &   90 &    \textbf{0 (0.00\%)} &    \textbf{0 (0.00\%)} &    \textbf{0 (0.00\%)} \\
D05 &   72 &             0 (0.00\%) &             2 (2.78\%) &    \textbf{5 (6.94\%)} \\
D06 &   90 &             0 (0.00\%) &    \textbf{2 (2.22\%)} &             0 (0.00\%) \\
D07 &   90 &    \textbf{0 (0.00\%)} &    \textbf{0 (0.00\%)} &    \textbf{0 (0.00\%)} \\
D08 &   90 &             0 (0.00\%) &    \textbf{4 (4.44\%)} &             0 (0.00\%) \\
D09 &   90 &             3 (3.33\%) &             0 (0.00\%) &    \textbf{7 (7.78\%)} \\
D10 &   90 &             0 (0.00\%) &    \textbf{4 (4.44\%)} &             3 (3.33\%) \\
\hline
E01 &   56 &             0 (0.00\%) &             0 (0.00\%) &    \textbf{2 (3.57\%)} \\
E02 &   56 &             4 (7.14\%) &  \textbf{14 (25.00\%)} &             4 (7.14\%) \\
E03 &   20 &   \textbf{6 (30.00\%)} &             0 (0.00\%) &            3 (15.00\%) \\
E04 &   20 &            2 (10.00\%) &            3 (15.00\%) &   \textbf{4 (20.00\%)} \\
E05 &   30 &            5 (16.67\%) &   \textbf{6 (20.00\%)} &            5 (16.67\%) \\
E06 &   90 &             4 (4.44\%) &           36 (40.00\%) &  \textbf{42 (46.67\%)} \\
E07 &   90 &             3 (3.33\%) &  \textbf{18 (20.00\%)} &             7 (7.78\%) \\
E08 &   90 &           12 (13.33\%) &           39 (43.33\%) &  \textbf{56 (62.22\%)} \\
E09 &   90 &           10 (11.11\%) &  \textbf{38 (42.22\%)} &           35 (38.89\%) \\
E10 &   90 &             6 (6.67\%) &             0 (0.00\%) &  \textbf{14 (15.56\%)} \\
\hline
L01 &   90 &             0 (0.00\%) &  \textbf{52 (57.78\%)} &           38 (42.22\%) \\
L02 &   90 &             1 (1.11\%) &  \textbf{15 (16.67\%)} &             8 (8.89\%) \\
L03 &   90 &    \textbf{0 (0.00\%)} &    \textbf{0 (0.00\%)} &    \textbf{0 (0.00\%)} \\
L04 &   90 &             0 (0.00\%) &    \textbf{4 (4.44\%)} &    \textbf{4 (4.44\%)} \\
L05 &   90 &             0 (0.00\%) &    \textbf{1 (1.11\%)} &             0 (0.00\%) \\
L06 &   90 &   \textbf{9 (10.00\%)} &             0 (0.00\%) &             6 (6.67\%) \\
L07 &   90 &  \textbf{11 (12.22\%)} &             5 (5.56\%) &             7 (7.78\%) \\
L08 &   90 &    \textbf{0 (0.00\%)} &    \textbf{0 (0.00\%)} &    \textbf{0 (0.00\%)} \\
L09 &   90 &             0 (0.00\%) &    \textbf{2 (2.22\%)} &             0 (0.00\%) \\
L10 &   90 &    \textbf{0 (0.00\%)} &    \textbf{0 (0.00\%)} &    \textbf{0 (0.00\%)} \\
\bottomrule
\end{tabular}
\end{table}

\begin{table}
\centering
\caption{Model comparison for each subcategory of bigger analogy test set when we excluded $a, b, c$ from the candidates.}
\label{tbl:eval-bats-detail-exclude-abc}
\begin{tabular}{lrlll}
\toprule
{} &  num &                     WV &               WV + MLP &               WV + AGN \\
\midrule
I01 &   90 &           53 (58.89\%) &             5 (5.56\%) &  \textbf{55 (61.11\%)} \\
I02 &   90 &  \textbf{42 (46.67\%)} &             0 (0.00\%) &           30 (33.33\%) \\
I03 &   90 &  \textbf{86 (95.56\%)} &           22 (24.44\%) &           73 (81.11\%) \\
I04 &   90 &  \textbf{68 (75.56\%)} &           18 (20.00\%) &           64 (71.11\%) \\
I05 &   90 &  \textbf{61 (67.78\%)} &           58 (64.44\%) &  \textbf{61 (67.78\%)} \\
I06 &   90 &           69 (76.67\%) &           13 (14.44\%) &  \textbf{71 (78.89\%)} \\
I07 &   90 &           52 (57.78\%) &           55 (61.11\%) &  \textbf{68 (75.56\%)} \\
I08 &   90 &           56 (62.22\%) &           42 (46.67\%) &  \textbf{69 (76.67\%)} \\
I09 &   90 &           58 (64.44\%) &           50 (55.56\%) &  \textbf{85 (94.44\%)} \\
I10 &   90 &           69 (76.67\%) &           61 (67.78\%) &  \textbf{80 (88.89\%)} \\
\hline
D01 &   90 &    \textbf{0 (0.00\%)} &    \textbf{0 (0.00\%)} &    \textbf{0 (0.00\%)} \\
D02 &   90 &    \textbf{3 (3.33\%)} &    \textbf{3 (3.33\%)} &             2 (2.22\%) \\
D03 &   90 &  \textbf{26 (28.89\%)} &             2 (2.22\%) &           15 (16.67\%) \\
D04 &   90 &  \textbf{11 (12.22\%)} &             0 (0.00\%) &             3 (3.33\%) \\
D05 &   72 &  \textbf{21 (29.17\%)} &             2 (2.78\%) &           17 (23.61\%) \\
D06 &   90 &           13 (14.44\%) &             2 (2.22\%) &  \textbf{19 (21.11\%)} \\
D07 &   90 &             1 (1.11\%) &             0 (0.00\%) &    \textbf{4 (4.44\%)} \\
D08 &   90 &             1 (1.11\%) &    \textbf{4 (4.44\%)} &             2 (2.22\%) \\
D09 &   90 &           21 (23.33\%) &             0 (0.00\%) &  \textbf{24 (26.67\%)} \\
D10 &   90 &             6 (6.67\%) &             4 (4.44\%) &  \textbf{12 (13.33\%)} \\
\hline
E01 &   56 &  \textbf{18 (32.14\%)} &             0 (0.00\%) &             5 (8.93\%) \\
E02 &   56 &             0 (0.00\%) &  \textbf{14 (25.00\%)} &             4 (7.14\%) \\
E03 &   20 &    \textbf{0 (0.00\%)} &    \textbf{0 (0.00\%)} &    \textbf{0 (0.00\%)} \\
E04 &   20 &             0 (0.00\%) &            3 (15.00\%) &   \textbf{4 (20.00\%)} \\
E05 &   30 &             0 (0.00\%) &   \textbf{6 (20.00\%)} &             2 (6.67\%) \\
E06 &   90 &             5 (5.56\%) &           33 (36.67\%) &  \textbf{47 (52.22\%)} \\
E07 &   90 &             3 (3.33\%) &           17 (18.89\%) &  \textbf{19 (21.11\%)} \\
E08 &   90 &             2 (2.22\%) &           40 (44.44\%) &  \textbf{53 (58.89\%)} \\
E09 &   90 &           12 (13.33\%) &  \textbf{38 (42.22\%)} &           33 (36.67\%) \\
E10 &   90 &  \textbf{43 (47.78\%)} &             0 (0.00\%) &           38 (42.22\%) \\
\hline
L01 &   90 &             7 (7.78\%) &           48 (53.33\%) &  \textbf{51 (56.67\%)} \\
L02 &   90 &             3 (3.33\%) &           15 (16.67\%) &  \textbf{17 (18.89\%)} \\
L03 &   90 &    \textbf{3 (3.33\%)} &             0 (0.00\%) &             2 (2.22\%) \\
L04 &   90 &             1 (1.11\%) &             3 (3.33\%) &    \textbf{5 (5.56\%)} \\
L05 &   90 &    \textbf{1 (1.11\%)} &    \textbf{1 (1.11\%)} &    \textbf{1 (1.11\%)} \\
L06 &   90 &             0 (0.00\%) &             0 (0.00\%) &    \textbf{1 (1.11\%)} \\
L07 &   90 &  \textbf{12 (13.33\%)} &             4 (4.44\%) &             2 (2.22\%) \\
L08 &   90 &  \textbf{27 (30.00\%)} &             0 (0.00\%) &           17 (18.89\%) \\
L09 &   90 &             2 (2.22\%) &    \textbf{5 (5.56\%)} &    \textbf{5 (5.56\%)} \\
L10 &   90 &    \textbf{8 (8.89\%)} &             1 (1.11\%) &             5 (5.56\%) \\
\bottomrule
\end{tabular}
\end{table}

Table \ref{tbl:eval-bats-detail} and \ref{tbl:eval-bats-detail-exclude-abc} summarizes the full results for each subcategory in the bigger analogy test set.
We can see the performance improvement of the proposed method in most subcategories.

\end{document}